\documentclass[10pt,twocolumn,letterpaper]{article}

\usepackage{iccv}
\makeatletter
\@namedef{ver@everyshi.sty}{}
\makeatother
\usepackage{times}
\usepackage{epsfig}
\usepackage{graphicx}
\usepackage{amsmath}
\usepackage{amssymb}
\usepackage{amsthm}
\usepackage{microtype}
\usepackage{subfigure}
\usepackage{tabularx}
\usepackage{amsfonts}
\usepackage{booktabs}
\usepackage{multirow}
\usepackage[percent]{overpic}
\usepackage{pdfpages}
\usepackage{pgffor}

\makeatletter
\AtBeginDocument{\let\LS@rot\@undefined}
\makeatother
\def\supplementfilename{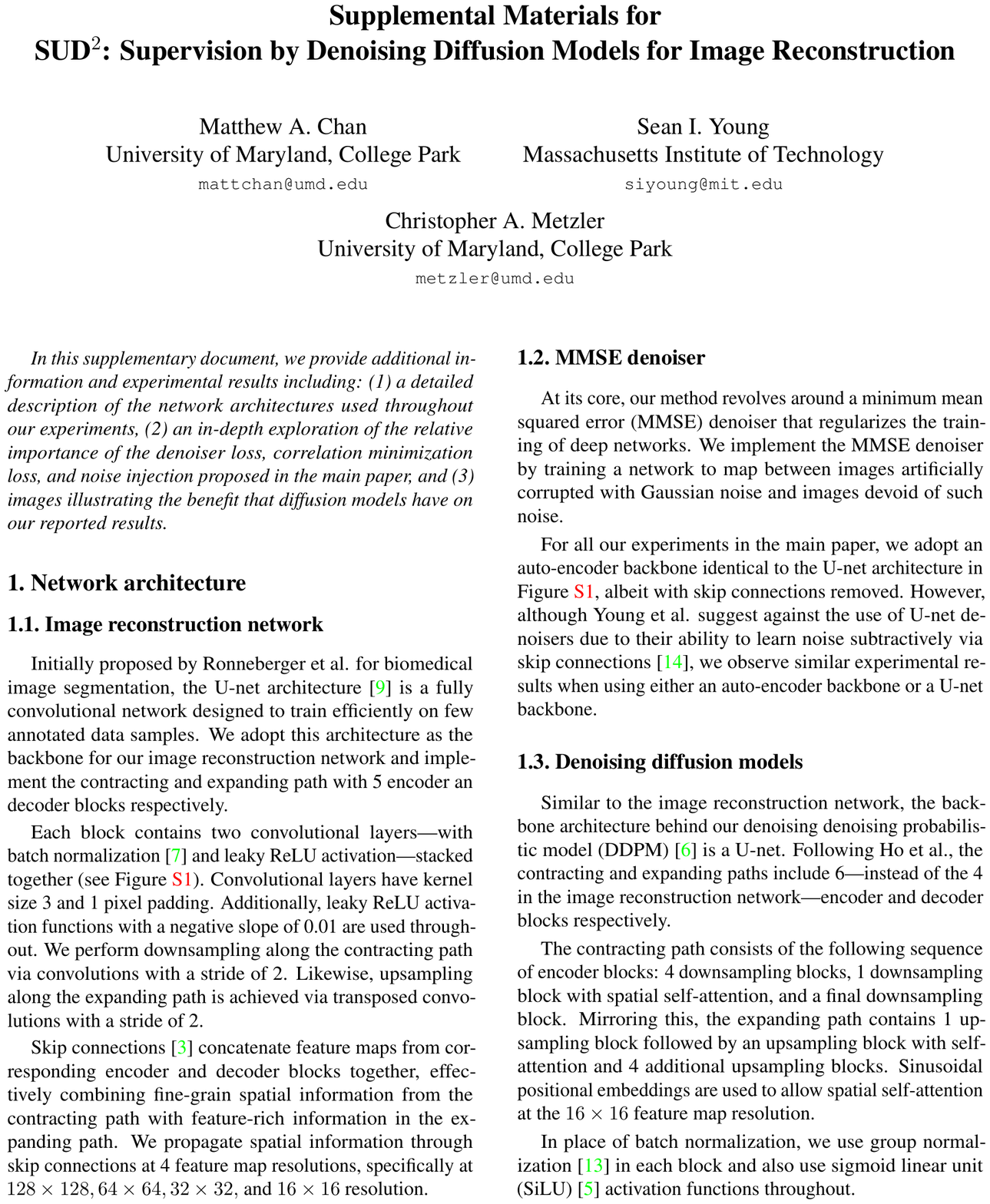}
\pdfximage{\supplementfilename}
\def\numbersupplementpages{\the\pdflastximagepages}
\newif\ifarXiv
\arXivtrue 
\usepackage{tikz}
\usetikzlibrary{arrows}
\usetikzlibrary{calc}
\usetikzlibrary{arrows.meta}
\usetikzlibrary{tikzmark}

\theoremstyle{plain}
\newtheorem{theorem}{Theorem}[section]

\newtheorem{corollary}[theorem]{Corollary}
\theoremstyle{definition}

\theoremstyle{remark}

\usepackage[pagebackref=true,breaklinks=true,colorlinks,bookmarks=false]{hyperref}

\iccvfinalcopy %
\def\iccvPaperID{11570} %

\ificcvfinal\fi

\begin{document}

\title{SUD$^2$: Supervision by Denoising Diffusion Models for Image Reconstruction}

\author{Matthew A. Chan\\
University of Maryland, College Park\\
{\tt\small mattchan@umd.edu}
\and
Sean I. Young\\
Massachusetts Institute of Technology\\
{\tt\small siyoung@mit.edu}
\and
Christopher A. Metzler\\
University of Maryland, College Park\\
{\tt\small metzler@umd.edu}
}
\date{}

\maketitle
\ificcvfinal\thispagestyle{empty}\fi

\begin{abstract}
	Many imaging inverse problems---such as image-dependent in-painting and dehazing---are challenging because their forward models are unknown or depend on unknown latent parameters. While one can solve such problems by training a neural network with vast quantities of paired training data, such paired training data is often unavailable. In this paper, we propose a generalized framework for training image reconstruction networks when paired training data is scarce. In particular, we demonstrate the ability of image denoising algorithms and, by extension, denoising diffusion models to supervise network training in the absence of paired training data.\\
    \vspace{-10pt}
\end{abstract}

\section{Introduction}

Imaging inverse problems can generally be described in terms of a forward operator $\mathcal{F}(\cdot)$ that maps a scene $x$  to a measurement $y$ according to
\begin{equation}
    y = \mathcal{F}(x).
\end{equation}
The goal of an image reconstruction algorithm is to recover $x$ from $y$. 

Historically, computational imaging research has focused on solving inverse problems with known forward models. For instance, computed tomography's forward model can be represented as a Radon transform, and magnetic resonance imaging's forward model can be represented as 2D Fourier Transform. Knowledge of these forward models allows one to reconstruct scenes $x$ from measurements $y$ using any number of classical or learning-based algorithms~\cite{ongie2020deep}.

Since the onset of the deep learning era, significant progress has been made in solving inverse problems which lack explicit forward models. By leveraging large amounts of training pairs $\{x_i, y_i\}^N_{i=0}$, neural networks learn to directly map samples from a source distribution, $y_i$, to images from a target distribution, $x_i$. In doing so, the network implicitly learns the inverse operator $\mathcal{F}^{-1}$ without any explicit knowledge of the forward model $\mathcal{F}$. 

The main drawback of deep learning methods is that their performance is directly related to the size and quality of the training dataset. As a result, these methods often struggle whenever little to no paired training data is available. 

\subsection{Problem setup}

Our goal in this work is to train a network $f_\theta(\cdot)$ to reconstruct images/scenes $x$ from measurements $y$ using three sets of training data.

\begin{itemize}
\item A small set $P$ of paired examples $(x_p,y_p)$ drawn from the joint distribution $p_{x,y}$.
\item  A large set $U_y$ of unpaired measurements $y_u$ drawn from the marginal distribution $p_y$.
\item  A large set $U_x$ of unpaired images $x_u$ drawn from the marginal distribution $p_x$.
\end{itemize}
Such mixed datasets naturally occur in applications where gathering unpaired data is easy, but gathering paired data is a challenge. For instance, it is straightforward to capture images with fog and images without fog, but capturing two paired images of the same scene with and without fog (with all lighting conditions and other nuisance variations fixed) is very challenging. Often times, the latter paired dataset is restricted to only a few images captured in a lab.

A paired training set, $P$, allows one to optimize $f_\theta(\cdot)$ by minimizing the empirical risk
\begin{align}
\mathcal{L}_{\text{paired}}=\frac{1}{|P|}\sum_{(x_p,y_p)\in P}\|x_p-f_\theta(y_p)\|^2,
\end{align}
where $|P|$ denotes the cardinality of $P$. 
However, as the size of $P$ decreases $\mathcal{L}_{\text{paired}}$ becomes a poor approximation of the true risk and $f_\theta(\cdot)$ overfits to the training set. As an alternative, we seek to leverage unpaired datasets $U_x$ and $U_y$ to improve the quality of our reconstructions.

\subsection{Our contributions}

In this work we introduce a novel semi-supervised learning framework for addressing image reconstruction problems, such as image dehazing, which lack an explicit forward model and for which gathering paired training data is challenging. Our central contributions are as follows;%
\begin{itemize}
    \item We generalize the supervision-by-denoising (SUD) semi-supervised learning technique from~\cite{sud} to work on any image reconstruction task, not just medical image segmentation.
    \item We prove that SUD implicitly performs cross-entropy minimization and use this connection to identify its various failure modes.
    \item Based on our analysis, we identify three techniques to improve SUD: (1) sample correlation minimization, (2) noise injection, and (3) denoising diffusion models.
    \item We demonstrate that the resulting algorithm, which we call SUD$^2$, outperforms existing semi-supervised and unsupervised learning techniques, such as CycleGAN, on image in-painting and image dehazing.

\end{itemize}

\section{Related work}
\label{sec:related}

\subsection{Semi-supervised learning}
\label{sec:ssl}
While crowdsourcing platforms such as Amazon's Mechanical Turk~\cite{AMT} can generate abundant labeled data for human-annotatable tasks like image segmentation, it's all but impossible to hand-label training data for many imaging inverse problems. 
 Given a foggy image, how would one generate corresponding clean image to use for training? 
 Even in applications where hand-labeling is possible, it is often prohibitively expensive to perform at scale.

Semi-supervised learning (SSL) serves as a workaround for applying deep learning techniques in situations where paired data is scarce. Typically, SSL methods regularize network training by leveraging information extracted from unpaired data and generally fall into two distinct categories; pseudo-labelling methods and consistency regularization methods~\cite{ssl_survey}. Pseudo-labelling methods~\cite{pseudolabel} help supervise training by generating fake labels on unpaired data samples, which allows the training of networks in a fully-supervised manner. In contrast, consistency regularization methods like temporal ensembling~\cite{temporal_ensembling} and mean teacher models~\cite{mean_teacher} enforce a regularization objective during training that makes the network more robust to perturbations in the data.

 Our method borrows ideas from both categories of SSL methods. Similar to pseudo-labelling methods, we generate fake labels on unpaired data by incorporating a denoising diffusion probabilistic model (DDPM)~\cite{ddpm} into the training pipeline. Additionally, following consistency regularization methods, we impose a regularization objective that encourages our image reconstruction network to produces diverse outputs.

\subsection{Regularization-by-denoising}
\label{sec:red}
Image denoising algorithms (denoisers) can serve as powerful priors on the form and distribution of natural images~\cite{egiazarian2007compressed,venkatakrishnan2013plug,metzler2016denoising}. 
Regularization-by-denoising (RED)~\cite{red,red_clarifications} is a regularization technique that leverages an image denoiser to solve classical single-image imaging inverse problems. Given an known forward model $\mathcal{F}$ and a single measurement $y$, RED (under a white Gaussian prior on the measurement noise) tries to reconstruct the unknown scene $x$ by minimizing
\begin{equation}
    \arg\min_x\|y-\mathcal{F}(x)\|^2+\rho(x) ,
\end{equation}
where
\begin{align}
\label{eq:red}
    \rho(x) = \frac{1}{2}x^T\left[x-D(x)\right],
\end{align}
and $D(\cdot)$ represents a denoiser like BM3D~\cite{dabov2007image} or DnCNN~\cite{zhang2017beyond}. Variations on this idea have recently been combined with diffusion models as well~\cite{chung2023diffusion}.

One simplified interpretation of the RED regularization objective~\eqref{eq:red} is that a good denoiser $D(\cdot)$ will project $x$ onto the manifold $\mathcal{M}$ of natural images. 
If $x$ is far from that manifold $\rho(x)$ will be large, whereas if $x$ is already close to this manifold, $D(x)$ will change it very little and $\rho(x)$ will be small.

Similar to RED, our method leverages an image denoiser. However, instead of using the denoiser to regularize classical inverse imaging optimization objectives, we follow~\cite{sud} and use an image denoiser to regularize the training of deep neural networks. That is, we are using denoiser to regularize a {\em function} rather than an image. The resulting function/reconstruction network, unlike the RED, is non-iterative and does not require explicit knowledge of the forward model $\mathcal{F}$.

\begin{figure*}[ht]
\centering
\includegraphics[width=\textwidth]{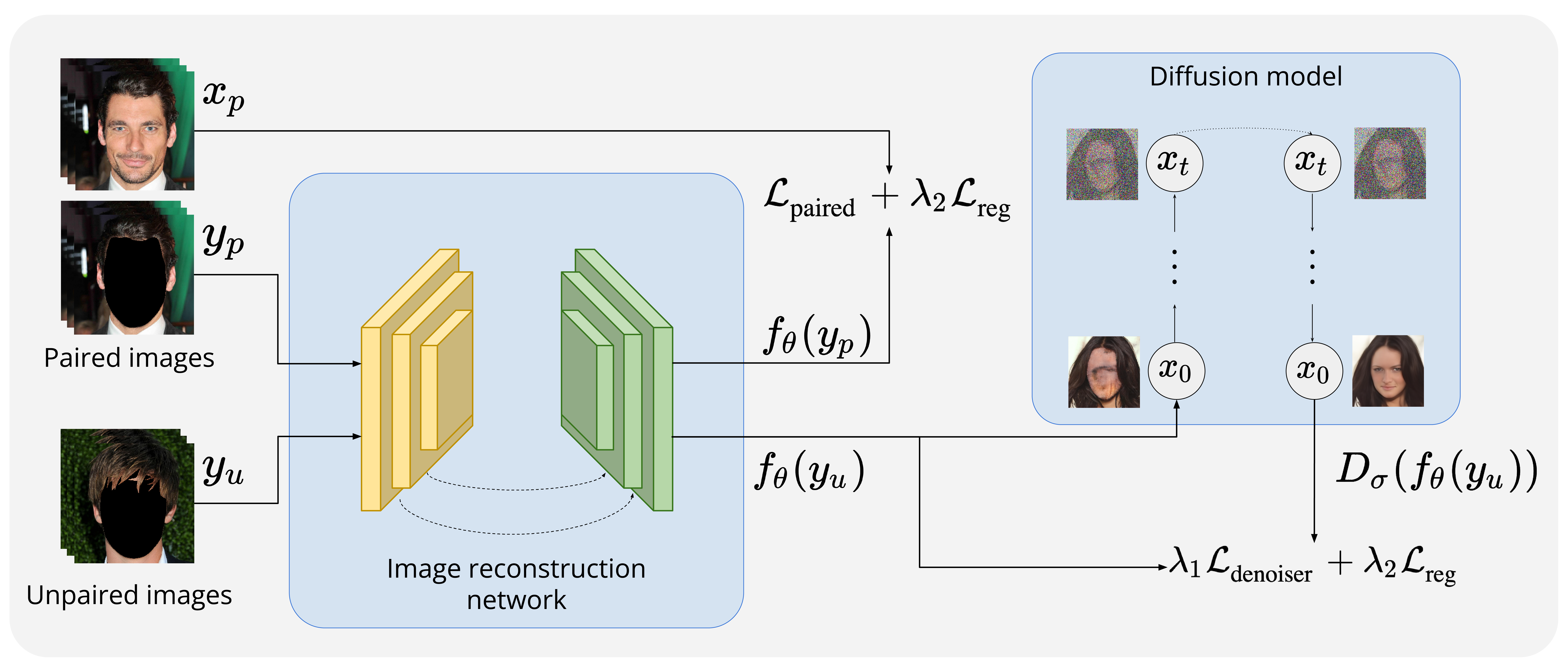}
\caption{{\bf Overview of the training pipeline.} Block diagram of the semi-supervised training pipeline used in our experiments. The pre-trained diffusion model supervises training by pushing outputs of the image reconstruction network towards the desired target image distribution. Note that only the image reconstruction network (a standard U-net in our experiments) is used at inference time.}
\label{ref:block_diagram}
\end{figure*}

\section{Methods}

\subsection{Supervision-by-denoising}

Young et al.~\cite{sud} recently introduced the supervision-by-denoising (SUD) framework which takes the ideas behind regularization-by-denoising~\cite{red} and extends them to enable semi-supervised learning. 
The intuition behind SUD is that a pretrained image denoiser, $D_\sigma(u)$, (which can be trained using the set $U_x$ of unpaired images) encodes strong priors on the distribution $p_x$. SUD enforces that the network's reconstructions $f_\theta(y_u)$ on the unpaired training data are ``consistent'' with the priors encoded in the denoiser.

When used in combination with an $\ell_2$ loss and without temporal-ensembling/damping, SUD effectively minimizes
\begin{align}
\mathcal{L}_{\text{paired}}+\lambda_1\mathcal{L}_{\text{denoiser}},
\end{align}
where $\lambda_1$ is a tuning parameter and
\begin{align}\label{eqn:denoiserloss}
    \mathcal{L}_{\text{denoiser}}=\frac{1}{|U_y|}\sum_{y_u\in U_y}\|f_\theta(y_u)-z_u \|^2,
\end{align}
where $z_u=D_\sigma(f_\theta(y_u))$. 
When updating the network weights $\theta$ to minimize~\eqref{eqn:denoiserloss}, SUD treats $z_u$ as a fixed pseudo-label and does not propagate gradients through the denoiser. That is, SUD defines the gradient of $\mathcal{L}_{\text{denoiser}}$ with respect to a single reconstruction $f_\theta(y_u)$ as
\begin{align}
\label{eqn:SUDGrads}
    \nabla_{f_\theta(y_u)}\mathcal{L}_{\text{denoiser}}=\frac{2[f_\theta(y_u)-D_\sigma(f_\theta(y_u)) ]}{|U_y|}.
\end{align}

As demonstrated in~\cite{sud}, SUD is a powerful and effective semi-supervised learning technique in the context of medical segmentation, where the goal is to map an image to a discrete-valued segmentation map. Using only a handful paired images and segmentation maps, Young et al.~were able to train a denoiser to segment brains, kidneys, and tumors.

Unfortunately, we found that without modification, SUD, with or without temporal ensembling, was far less effective at general image restoration tasks. As illustrated in Figure~\ref{fig:mode_collapse}, minimizing the SUD loss for CelebA face inpainting leads to mode collapse.

\subsection{Understanding SUD}

In this section we analyze SUD in order to identify and overcome its weaknesses.

\begin{theorem}
\label{thm:CrossEntropy}
When $D_\sigma$ is a minimum-mean-squared error (MMSE) Gaussian denoiser, minimizing $\mathcal{L}_{\text{denoiser}}$ minimizes the cross entropy between the distribution of $f_\theta(y_u)$ and the smoothed version of $p_x$. 
\end{theorem}

\begin{proof}
Let $\nu$ follow an independent zero-mean white Gaussian distribution with variance $\sigma^2$. 
We will use $\nu$ to smooth the distributions $p_x$ (recall $p_{x+\nu}=p_x*p_{\nu}$, where $*$ denotes convolution) so that we can take advantage of Tweedie's Formula, as described below.%

The cross entropy between $p_{f_\theta(y)}$ and $p_{x+\nu}$ is, by definition, 
\begin{align}
\label{eqn:CrossEntropy}
H(p_{f_\theta(y)},p_{x+\nu})&=-\mathbb{E}_{f_\theta(y)}[\ln p_{x+\nu}(f_\theta(y))].
\end{align}

We can form a Monte-Carlo approximate of the expectation in~\eqref{eqn:CrossEntropy} by averaging over $U_y$:
\begin{align}
	H(p_{f_\theta(y)},p_{x+\nu})\approx - \frac{1}{|U_y|}\sum_{y_u\in U_y} \ln p_{x+\nu}(f_\theta(y_u)).
\end{align}

Then, we can express the gradient of this loss with respect to a reconstruction $f_\theta(y_u)$ as
\begin{align}
\label{eqn:lossgrad1}
	&\nabla_{f_\theta(y_u)} H(p_{f_\theta(y)},p_{x+\nu})\nonumber\\
 &\approx - \frac{\nabla_{f_\theta(y_u)} \ln p_{x+\nu}(f_\theta(y_u))}{|U_y|}.
\end{align}

To efficiently evaluate~\eqref{eqn:lossgrad1} we turn to Tweedie's Formula. Tweedie's Formula~\cite{tweedie} states that for a signal corrupted with zero-mean additive white Gaussian noise, $r=x+\nu$ where $\nu\sim\mathcal{N}(0,\sigma^2\mathbf{I})$, the output of a MMSE denoiser $D_{\sigma}(\cdot)$ (and by extension a neural network trained to act as a MMSE denoiser) can be expressed as
\begin{equation}
    D_{\sigma}(r)=r+\sigma^2\nabla_r\ln p_{x+\nu}(r).    
    \label{eq:tweedie}
\end{equation}

In other words, denoisers perform gradient ascent on the log-likelihood of $p_{x+\nu}$ %
where the step size corresponds to the noise variance. Accordingly, we can express the gradient of the log-likelihood in terms of the denoiser's residual:
\begin{align}
	\nabla_r\ln p_{x+\nu}(r)=\frac{D_{\sigma}(r)-r}{\sigma^2}.
\end{align}

By applying Tweedie's formula to~\eqref{eqn:lossgrad1} we arrive at
\begin{align}
\label{eqn:lossgradDenoiser}
	\nabla_{f_\theta(y_u)} H(p_{f_\theta(y)},p_{x+\nu})\approx \frac{[f_\theta(y_u)-D_\sigma(f_\theta(y_u))]}
	{\sigma^2|U_y|}.
\end{align}

Up to constants, this is the same expression as the SUD gradients defined in~\eqref{eqn:SUDGrads}. As such, the SUD denoiser loss minimizes the cross entropy between $p_{f_\theta(y)}$ and $p_{x+\nu}$.
\end{proof}

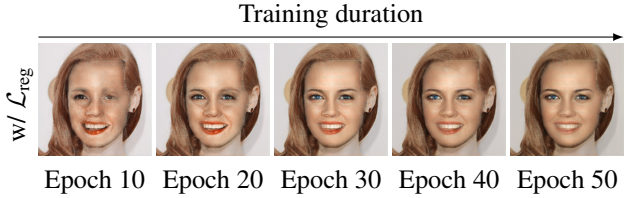
\begin{figure}
    \centering
    \vskip 0.3in
    \setlength\tabcolsep{1pt}
    \begin{tabular}{cccccc}
      \multirow{1}{*}[0.9cm]{\rotatebox{90}{w/ $\mathcal{L}_{\text{reg}}$}}&\tikzmark{start}\begin{overpic}[width=0.18\columnwidth]{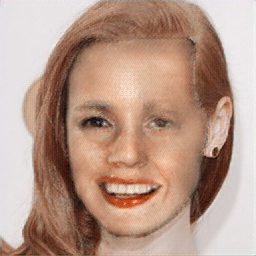}\end{overpic} &
      \begin{overpic}[width=0.18\columnwidth]{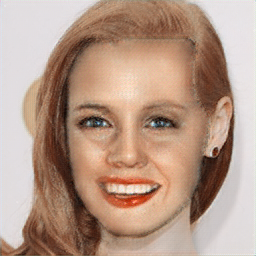}\end{overpic} &
      \begin{overpic}[width=0.18\columnwidth]{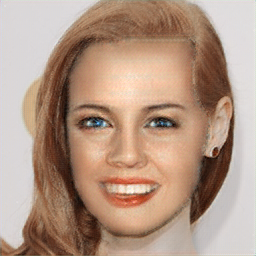}\end{overpic} &
      \begin{overpic}[width=0.18\columnwidth]{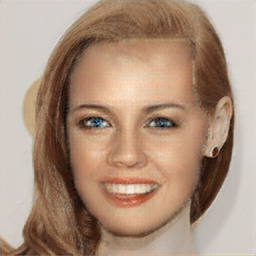}\end{overpic} &
      \begin{overpic}[width=0.18\columnwidth]{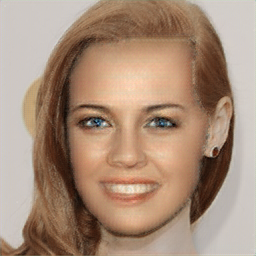}\end{overpic}\tikzmark{end}\\
      & Epoch 10 & Epoch 20 & Epoch 30 & Epoch 40 & Epoch 50
    \end{tabular}
    \begin{tikzpicture}[remember picture,overlay]
    \draw[-latex] ([ shift={(0,10ex)}]pic cs:start) -- node[above] {Training duration} ([ shift={(0,10ex)}]pic cs:end);
    \end{tikzpicture}
    \vskip 0.1in
    \caption{{\bf SUD with and without correlation minimization.} Without a correlation loss to regularize reconstructions, the SUD denoising objective causes network outputs to converge to a mode---in this case a washed out image.}
    \label{fig:mode_collapse}
\end{figure}

\begin{corollary}
    \label{corr:mode_collapse}
    Minimizing $\mathcal{L}_{\text{denoiser}}$ encourages mode collapse.
\end{corollary}
\begin{proof}
Minimizing $\mathcal{L}_{\text{denoiser}}$ minimizes the cross entropy between $p_{f_\theta(y_u)}$ and $p_{x+\nu}$. The cross entropy of $H(p,q)$ of two distributions $p$ and $q$ is minimized with respect to $p$ when $p$ is a dirac distribution with a non-zero support where distribution $q$ is largest, i.e.,~a mode.
\end{proof}

Examples of SUD leading to mode collapse are presented in Figure~\ref{fig:mode_collapse}.

\begin{corollary}
    \label{corr:blurry_recons}
    Minimizing $\mathcal{L}_{\text{denoiser}}$ can encourage blurry reconstructions.
\end{corollary}
\begin{proof}
Minimizing $\mathcal{L}_{\text{denoiser}}$ minimizes the cross entropy between $p_{f_\theta(y_u)}$ and $p_{x+\nu}$ and will result in solutions $f_\theta(y_u)$ that maximize $p_{x+\nu}(f_\theta(y_u))$. %
For sufficiently large $\sigma$, $p_{x+\nu}$ is maximized not where $p_x$ is large (along the manifold of natural images) but rather at some point in between high-probability points. 
\end{proof}

A toy example illustrating how introducing noise onto a random variable can move its distribution's maxima is presented in Figure~\ref{fig:ToyNoise}.

\begin{figure}[t]
\centering
\setlength\tabcolsep{1pt}
\begin{tabular}{ccc}
  \begin{overpic}[width=0.33\columnwidth]{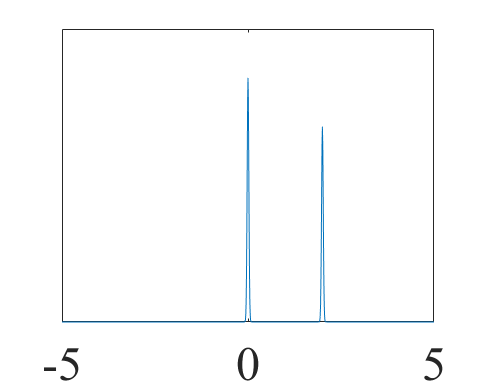}\end{overpic} &
  \begin{overpic}[width=0.33\columnwidth]{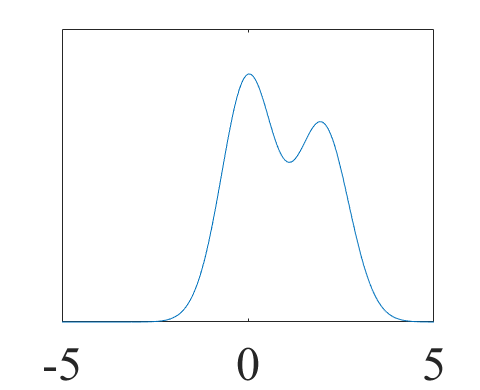}\end{overpic} &
  \begin{overpic}[width=0.33\columnwidth]{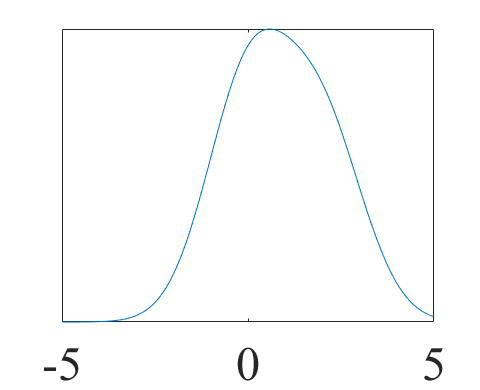}\end{overpic}\\
  $p_x$ & $p_{x+\nu_1}$ &  $p_{x+\nu_2}$
\end{tabular}
\vskip 0.1in
\caption{{\bf PDFs of a bimodal distribution with varying amounts of noise.} Adding sufficient noise can move the distribution's mode.}
\label{fig:ToyNoise}
\end{figure}

\subsection{Improving SUD}
To fight mode collapse, we introduce an additional penalty, $\mathcal{L}_{\text{reg}}$ into the SUD loss which encourages diverse outputs. That is, we minimize
\begin{align}
    \mathcal{L}_{\text{paired}}+\lambda_1\mathcal{L}_{\text{denoiser}}+\lambda_2\mathcal{L}_{\text{reg}}
\end{align}
where $\lambda_1$ and $\lambda_2$ are scalar weights on each loss term. 

\begin{theorem}
    When no paired training data is present, SUD with $\lambda_1=\frac{1}{2\sigma^2}$, $\lambda_2=1$, and $\mathcal{L}_{\text{reg}}=-H(p_{f_\theta(y)})$ minimizes the KL divergence between the distributions of $f_\theta(y)$ and $x+\nu$.
\end{theorem}

\begin{proof}
Our results from Theorem~\ref{thm:CrossEntropy} indicate  $\frac{1}{2\sigma^2}\mathcal{L}_{\text{denoiser}}$ and $H(p_{f_\theta(y)},p_{x+\nu})$ have the same gradients. Thus minimizing 
\begin{align}
    \frac{1}{2\sigma^2}\mathcal{L}_{\text{denoiser}}+\mathcal{L}_{\text{reg}}
\end{align}
minimizes 
\begin{align}
 H(p_{f_\theta(y)},p_{x+\nu})-H(p_{f_\theta(y)}).
 \end{align}
The latter expression is the definition of KL-divergence between the distributions of $f_\theta(y)$ and $x+\nu$.
\end{proof}

\subsubsection{Correlation minimization}
Computing and maximizing the entropy $H(p_{f_\theta(y_u)})$ is computationally intractable. 
To get around this hurdle, we instead encourage sample diversity as follows: 
(1) we use a U-net architecture as our reconstruction network $f_\theta(\cdot)$,
(2) we let $a_u$ represent the intermediate activation of $f_\theta$ to inputs $y_u$, and
(3) we then penalize correlations between activations. 

Specifically, in each mini-batch, we compute the normalized covariance matrix on the intermediate outputs from the encoder block of our network. Diagonal entries of the matrix---which contain the correlation coefficient of each vector with itself---are all equal to 1 by definition. Non-diagonal entries of the matrix contain the Pearson correlation coefficient (PCC) between the latent vectors in a mini-batch. 

PCC identifies linear correlations between the latent vectors, where a value of 0 signifies linearly uncorrelated samples, a value 1 signifies a positive linear correlation between samples, and a value of $-1$ signifies a negative linear correlation between samples. By minimizing this value for each entry in the correlation matrix, we encourage network to produce outputs which are uncorrelated in latent space.

\subsubsection{Noise injection}
\begin{figure}[t]
\setlength\tabcolsep{2pt}
\begin{tabular}{cccc}
  \multirow{1}{*}[10ex]{\rotatebox{90}{with noise}} &
  \begin{overpic}[width=0.29\columnwidth]{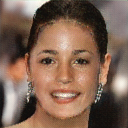}\end{overpic} &
  \begin{overpic}[width=0.29\columnwidth]{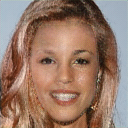}\end{overpic} &
  \begin{overpic}[width=0.29\columnwidth]{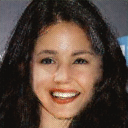}\end{overpic}\\
  \multirow{1}{*}[12ex]{\rotatebox{90}{without noise}} &
  \begin{overpic}[width=0.29\columnwidth]{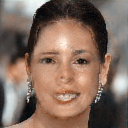}\end{overpic} &
  \begin{overpic}[width=0.29\columnwidth]{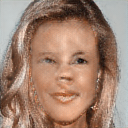}\end{overpic} &
  \begin{overpic}[width=0.29\columnwidth]{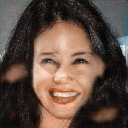}\end{overpic} 
\end{tabular}
\vskip 0.1in
\caption{{\bf SUD with and without noise injection.} Injecting noise onto the reconstructed images before denoising them results in more realistic-looking reconstructions.}
\label{fig:noise_ablation}
\end{figure}

SUD compares the distribution of the reconstructions $f_\theta(y)$ with the distribution $p_{x+\nu}$. As noted in the previous section, this encourages solutions for which $p_{x+\nu}$ is large but $p_x$ is small: i.e., it can produce solutions off of the manifold of natural images. 

To alleviate this problem, we inject noise onto the reconstructions $f_\theta(y_u)$ before passing them through the denoiser. That is, we redefine $\nabla_{f_\theta(y_u)}\mathcal{L}_{\text{denoiser}}$ as
\begin{align}
\label{eqn:SUDGrads2}
    \nabla_{f_\theta(y_u)}\mathcal{L}_{\text{denoiser}}=\frac{2[f_\theta(y_u)-D_\sigma(f_\theta(y_u)+\nu_2) ]}{|U_y|},
\end{align}
where $\nu_2\sim N(0,\sigma_2^2\mathbf{I})$.

This simple modification allows us to compare the smoothed distribution $p_{f_\theta(y)+\nu_2}$ with the smoothed distribution $p_{x+\nu}$. As demonstrated in Figure~\ref{fig:noise_ablation}, noise injection produces reconstructions of considerably higher quality.

\subsubsection{Diffusion models}
\label{sec:diffusion_models}
\begin{figure}[t]
\centering
\setlength\tabcolsep{2pt}
\begin{tabular}{cccc}
  \multirow{1}{*}[8ex]{\rotatebox[]{90}{DDPM}} &
  \begin{overpic}[width=0.29\columnwidth]{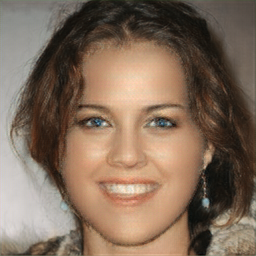}\end{overpic} &
  \begin{overpic}[width=0.29\columnwidth]{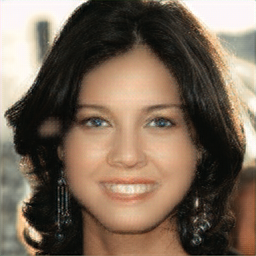}\end{overpic} &
  \begin{overpic}[width=0.29\columnwidth]{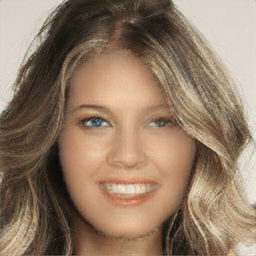}\end{overpic}\\
  \multirow{1}{*}[10ex]{\rotatebox[]{90}{AE}} &
  \begin{overpic}[width=0.29\columnwidth]{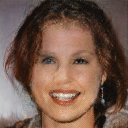}\end{overpic} &
  \begin{overpic}[width=0.29\columnwidth]{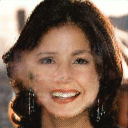}\end{overpic} &
  \begin{overpic}[width=0.29\columnwidth]{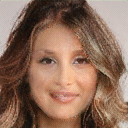}\end{overpic}
  \end{tabular}
\vskip 0.1in
\caption{{\bf SUD with and without denoising diffusion models.} Image in-painting samples generated by SUD with a multi-step denoising diffusion model (top) and a one-step autoencoder image denoiser (bottom). SUD training with the denoising diffusion model produces higher-quality reconstructions.}
\label{fig:diffusion_ablation}
\end{figure}

As mentioned in Section~\ref{sec:red}, an alternative and more heuristic interpretation of denoising algorithms is that they are projecting the reconstructions onto a manifold $\mathcal{M}$ of allowable reconstructions, e.g.,~faces or ``natural images''. Traditional denoising algorithms perform this projection in a single step. However, existing theory~\cite{wakin2005high} suggests that one should navigate image manifolds gradually, in a smooth-to-rough/coarse-to-fine manner.

Loosely inspired by this observation, we propose replacing our single-step MMSE denoising algorithm with a multi-step denoising diffusion probabilistic model (DDPMs)~\cite{ddpm}. 
That is, we replace our denoiser $D_\sigma(\cdot)$ used in~\eqref{eqn:SUDGrads2} with an iterative forward ``noising'' operator $F(\cdot)$ and an iterative reverse ``denoising'' operator $R(\cdot)$ such that $D_\sigma(f_{\theta}(y_u)+\nu)=R(F(f_{\theta}(y_u)))$. The definitions for each operator are expressed as
\begin{align}
    \label{eqn:DDPM}
    &\resizebox{0.9\columnwidth}{!}{$F(r)=\sqrt{\alpha_t}\left(\ldots\left(\sqrt{\alpha_0}r+(1-\alpha_0)z_0\right)+\ldots\right)+(1-\alpha_t)z_t$}\nonumber\\
    &\resizebox{0.97\columnwidth}{!}{$R(F(r))=\frac{1}{\sqrt{\alpha_0}}\left(\ldots\left(\frac{1}{\sqrt{\alpha_t}}\left(F(r) - \frac{1-\alpha_t}{\sqrt{1-\bar{\alpha}_t}}\epsilon_{\tau}\right)+\sqrt{1-\alpha_t}z_t\ldots\right)-\frac{1-\alpha_0}{\sqrt{1-\bar{\alpha}_0}}\epsilon_{\tau}\right)$}
\end{align}
where $\alpha$ controls the noise variance at each time step $t$, $z\sim\mathcal{N}(0,I)$, $\epsilon$ represents a diffusion network with weights $\tau$, and $\bar{\alpha}_t=\prod^t_i\alpha_i$.

Conceptually, the DDPM serves to first project $r$ onto the smooth manifold of noisy images and then gradually project $r$ onto correspondingly less smooth manifolds of less noisy images. 
As demonstrated in Figure~\ref{fig:diffusion_ablation}, training with a DDPM produces higher quality results than training with a denoising auto-encoder.

\subsection{SUD$^2$}
We refer to SUD, including all the aforementioned modifications, as SUD$^2$. Ablation studies characterizing the relative importance of correlation minimization, noise injection, and denoising diffusion models are provided in the supplement.

\section{Experiments}
\label{sec:experiments}

In the following experiments, we evaluate our method both quantitatively and qualitatively against baselines. All methods below are trained on 5 paired images, and the semi-supervised methods are trained on additional unpaired images. 

The backbone architecture for our image reconstruction network is a U-net~\cite{unet} consisting of 4 down-sampling and up-sampling blocks---implemented using strided convolutions---with skip connections between them. Each down/up-sampling block contains two convolutional layers, each with batch normalization and a Leaky ReLU activation function~\cite{relu}. The denoising diffusion models used in our experiments have a similar U-net architecture, albeit with 6 down/up-sampling blocks and spatial self-attention. Similarly, our denoising networks use an autoencoder backbone identical to the image reconstruction U-net, but with skip connections removed. 

We train all of the image reconstruction networks on $4\times$ Nvidia RTX A5000 GPUs using an Adam optimizer with an initial learning rate of $1\times10^{-3}$, a weight decay of $1\times10^{-4}$, and a batch size of 8 for 50 epochs. Additionally, we resize images to $256\times256$ resolution and normalize pixel intensities between $[-1,1]$ before training.

When evaluating our method on the experiments below, we set $\lambda_1=0.01, \lambda_2=10$ and supervise training by walking 400 steps along the forward diffusion process before walking 400 steps along the reverse diffusion process. Generally, we find $[300,700]$ steps to be a reasonable range to take. Traversing more than 700 steps along the diffusion chain adds excess amounts of noise to the image, resulting in an image which no longer resembles the input image. In contrast, traversal of less than 300 steps yields images which are approximately identical to the input images---making them poor targets for supervision.

Instead of explicitly setting the gradient of the loss, we equivalently implement~\eqref{eqn:SUDGrads2} by disabling backpropagation of gradients through the diffusion model $D_{\sigma}(\cdot)$ before computing the $\ell_2$ loss between $f_{\theta}(y_u)$ and $D_{\sigma}(f_{\theta}(y_u))$. We also add a perceptual loss---specifically Learned Perceptual Image Patch Similarity (LPIPS)~\cite{lpips}---to $\mathcal{L}_{\text{paired}}$ and $\mathcal{L}_{\text{denoiser}}$ as we empirically find it to help produce better quality reconstructions.

\subsection{Image in-painting}
\label{sec:inpainting}

\begin{figure*}[ht]
\centering
\subfigure{\includegraphics[width=0.29\columnwidth]{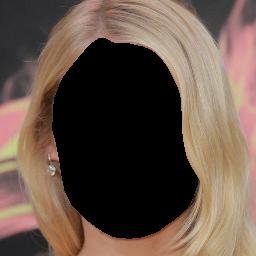}}
\subfigure{\begin{overpic}[width=0.29\columnwidth]{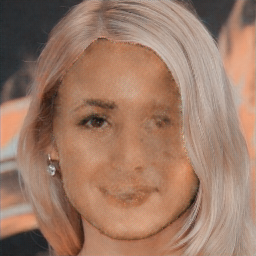}\end{overpic}}
\subfigure{\begin{overpic}[width=0.29\columnwidth]{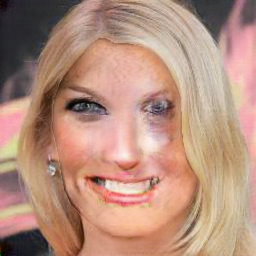}\end{overpic}}
\subfigure{\begin{overpic}[width=0.29\columnwidth]{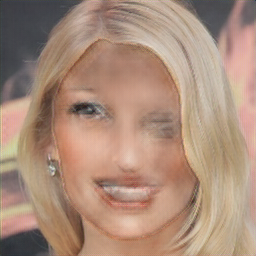}\end{overpic}}
\subfigure{\begin{overpic}[width=0.29\columnwidth]{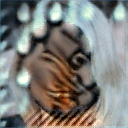}\end{overpic}}
\subfigure{\begin{overpic}[width=0.29\columnwidth]{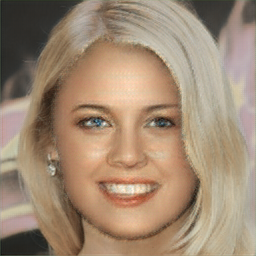}\end{overpic}}
\subfigure{\includegraphics[width=0.29\columnwidth]{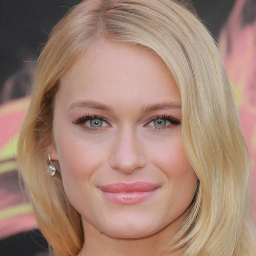}}\\
\vskip -0.1in
\subfigure{\includegraphics[width=0.29\columnwidth]{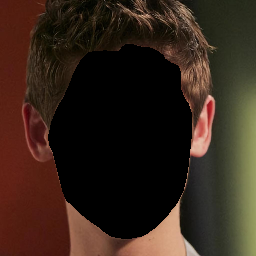}}
\subfigure{\begin{overpic}[width=0.29\columnwidth]{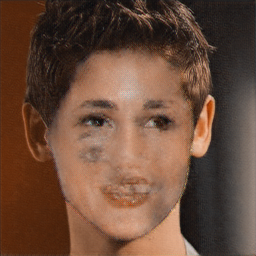}\end{overpic}}
\subfigure{\begin{overpic}[width=0.29\columnwidth]{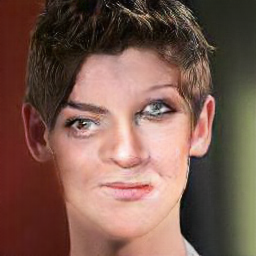}\end{overpic}}
\subfigure{\begin{overpic}[width=0.29\columnwidth]{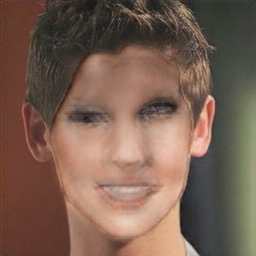}\end{overpic}}
\subfigure{\begin{overpic}[width=0.29\columnwidth]{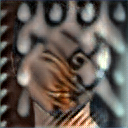}\end{overpic}}
\subfigure{\begin{overpic}[width=0.29\columnwidth]{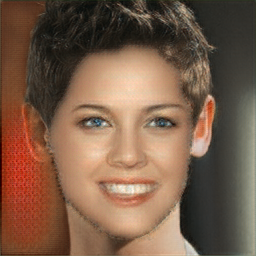}\end{overpic}}
\subfigure{\includegraphics[width=0.29\columnwidth]{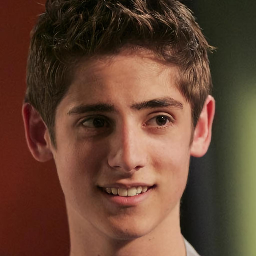}}\\
\vskip -0.1in
\setcounter{subfigure}{0}
\subfigure[Input]{\includegraphics[width=0.29\columnwidth]{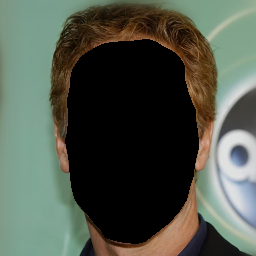}}
\subfigure[Supervised]{\begin{overpic}[width=0.29\columnwidth]{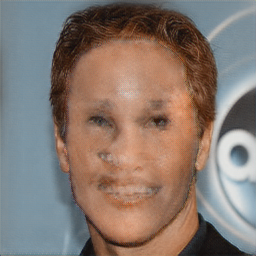}\end{overpic}}
\subfigure[CycleGAN]{\begin{overpic}[width=0.29\columnwidth]{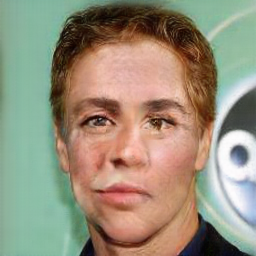}\end{overpic}}
\subfigure[CycleGAN-SSL]{\begin{overpic}[width=0.29\columnwidth]{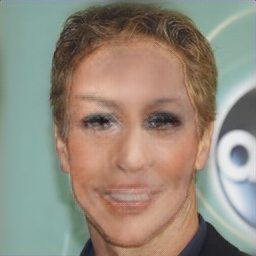}\end{overpic}}
\subfigure[SUD]{\begin{overpic}[width=0.29\columnwidth]{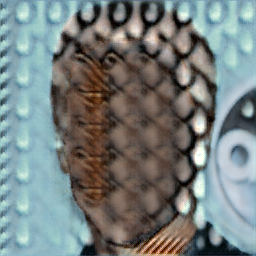}\end{overpic}}
\subfigure[$\text{SUD}^2$]{\begin{overpic}[width=0.29\columnwidth]{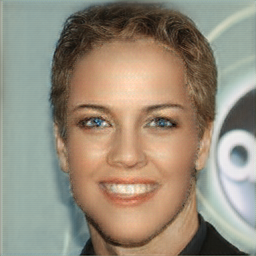}\end{overpic}}
\subfigure[Reference]{\includegraphics[width=0.29\columnwidth]{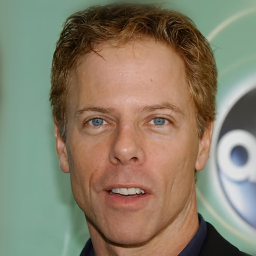}}\\
\vskip 0.1in
\caption{{\bf Qualitative in-painting results.} The images above show inference results from several methods on CelebAMask-HQ in-painting. All methods are trained on 5 pairs of masks and faces. Semi-supervised methods (CycleGAN-SSL, SUD, SUD$^2$) are trained on an additional 1000+12,500 unpaired faces.}
\label{ref:comparison_table}
\end{figure*}

Image in-painting is a generative process for reconstructing missing regions of an image such that restored image fits a desired---often natural---image distribution. We test our method on the CelebAMask-HQ dataset~\cite{CelebAMask-HQ}, which contains 30,000 images of $512\times512$ resolution and their corresponding segmentation maps. This dataset allows us to learn a strong prior since the image distribution is relatively constrained---all images are one-quarter head shots of celebrities. 

In this experiment, we mask out the subject's face from each image and train a few shot, semi-supervised in-painting network on 5 paired images and 1000+12,500 unpaired images. The 12,500 unpaired images are used to pre-train a blind Gaussian denoiser, a diffusion model, and a CycleGAN~\cite{CycleGAN2017}. Note that the denoiser and diffusion model are only trained on 12,500 faces without a mask applied. However, we train the CycleGAN on non-corresponding pairs of 12,500 faces without a mask applied and 12,500 faces with a mask applied to give it the best possible performance. As an additional baseline, we also train an image reconstruction network using the pre-trained CycleGAN as a pseudo-label generator, which we refer to as CycleGAN-SSL.

Despite requiring far fewer images than both CycleGAN baselines, SUD$^2$ produces faces which most closely resemble the ground truth distribution. Furthermore, the strong prior imposed by the pre-trained diffusion model results in reconstructions with more defined facial structure compared to baseline methods. Notably, as described in Corollary~\eqref{corr:mode_collapse}, the SUD baseline collapses to a mode during training with high probability, yielding highly correlated reconstructions. 

Quantitatively, SUD$^2$ achieves the most consistent results across our test set of $768$ images, with a $48\%$ higher average PSNR over CycleGAN-SSL and a $43\%$ lower average FID score compared to the supervised baseline. Although CycleGAN achieves the best FID score overall, it tends to produce qualitatively poor faces with disproportionately sized features, which is reflected in its poor PSNR, SSIM, and LPIPS scores. Likewise, while the supervised baseline achieves PSNR, SSIM, and LPIPS scores comparable to SUD$^2$, it often generates faces with missing features (i.e. eyes, nose, mouth), which is indicated by its high FID score.

\begin{figure}[t]
\centering
\begin{tabular}{lcccc}
  \hline
  \multicolumn{5}{c}{\textbf{Image in-painting}}\\
  \hline
  Method & PSNR$\uparrow$ & SSIM$\uparrow$ & LPIPS$\downarrow$ & FID$\downarrow$\\
  \hline
  Supervised & 18.44 & 0.71 & 0.29 & 0.48\\
  CycleGAN & 8.77 & 0.23 & 0.66 & \textbf{0.17}\\
  CycleGAN-SSL & 11.38 & 0.60 & 0.39 & 1.14\\
  SUD & 11.28 & 0.29 & 0.69 & 3.07\\
  $\text{SUD}^2$ (Ours)& \textbf{18.71} & \textbf{0.71} & \textbf{0.28} & 0.31\\
  \hline\\
\end{tabular}
\caption{\textbf{Quantitative image in-painting results.} This table lists the average scores attained by each method on the CelebAMask-HQ image in-painting test set. The best scores are bolded for readability.}
\label{table:inpainting_metrics}
\end{figure}

\begin{figure}[t]
\centering
\begin{tabular}{lccc}
  \hline
  \multicolumn{4}{c}{\textbf{Image dehazing}}\\
  \hline
  Method & PSNR$\uparrow$ & SSIM$\uparrow$ & LPIPS$\downarrow$\\
  \hline
  Supervised & 17.65 & \textbf{0.66} & 0.44\\
  CycleGAN & 12.80 & 0.22 & 0.64\\
  $\text{SUD}^2$ (Ours)& \textbf{17.74} & 0.64 & \textbf{0.37}\\
  \hline\\
\end{tabular}
\caption{\textbf{Quantitative image dehazing results.} This table lists the median scores attained by each method on the REVIDE image dehazing test set. The best scores are bolded for readability.}
\label{table:dehazing_metrics}
\end{figure}

\begin{figure*}[ht]
\centering
\subfigure{\includegraphics[width=0.40\columnwidth]{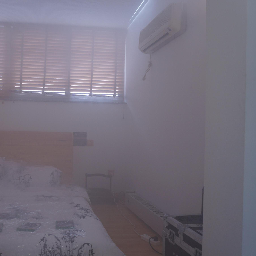}}
\subfigure{\begin{overpic}[width=0.40\columnwidth]{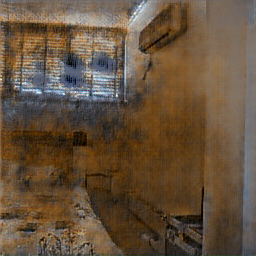}
    \put (5,5) {\scriptsize\textcolor{white}{\texttt{PSNR:16.53}}}
    \put (55,5) {\scriptsize\textcolor{white}{\texttt{SSIM:0.58}}}
\end{overpic}}
\subfigure{\begin{overpic}[width=0.40\columnwidth]{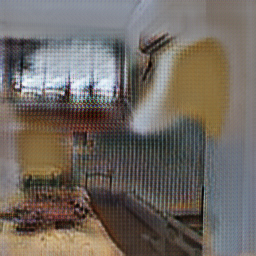}
    \put (5,5) {\scriptsize\textcolor{white}{\texttt{PSNR:15.68}}}
    \put (55,5) {\scriptsize\textcolor{white}{\texttt{SSIM:0.42}}}
\end{overpic}}
\subfigure{\begin{overpic}[width=0.40\columnwidth]{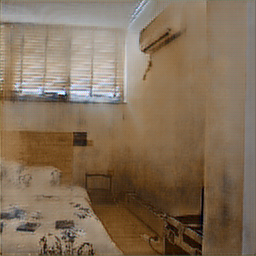}
    \put (5,5) {\scriptsize\textcolor{white}{\texttt{PSNR:17.76}}}
    \put (55,5) {\scriptsize\textcolor{white}{\texttt{SSIM:0.66}}}
\end{overpic}}
\subfigure{\includegraphics[width=0.40\columnwidth]{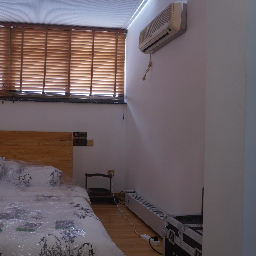}}\\
\vskip -0.1in
\subfigure{\includegraphics[width=0.40\columnwidth]{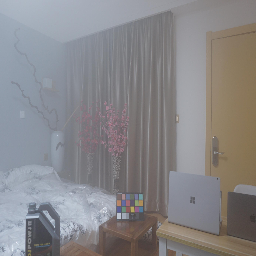}}
\subfigure{\begin{overpic}[width=0.40\columnwidth]{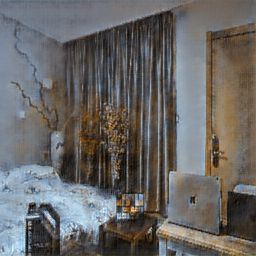}
    \put (5,5) {\scriptsize\textcolor{white}{\texttt{PSNR:21.70}}}
    \put (55,5) {\scriptsize\textcolor{white}{\texttt{SSIM:0.75}}}
\end{overpic}}
\subfigure{\begin{overpic}[width=0.40\columnwidth]{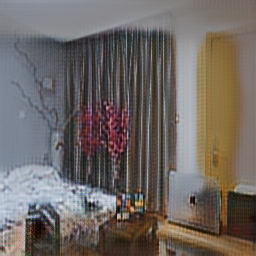}
    \put (5,5) {\scriptsize\textcolor{white}{\texttt{PSNR:19.66}}}
    \put (55,5) {\scriptsize\textcolor{white}{\texttt{SSIM:0.70}}}
\end{overpic}}
\subfigure{\begin{overpic}[width=0.40\columnwidth]{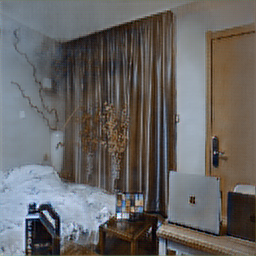}
    \put (5,5) {\scriptsize\textcolor{white}{\texttt{PSNR:22.01}}}
    \put (55,5) {\scriptsize\textcolor{white}{\texttt{SSIM:0.76}}}
\end{overpic}}
\subfigure{\includegraphics[width=0.40\columnwidth]{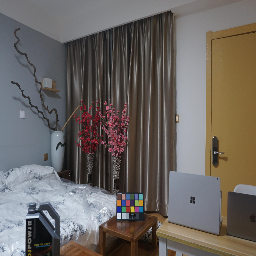}}\\
\vskip -0.1in
\setcounter{subfigure}{0}
\subfigure[Input]{\includegraphics[width=0.40\columnwidth]{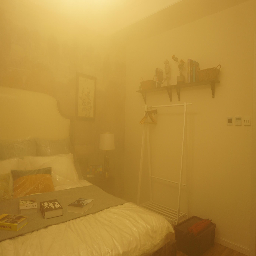}}
\subfigure[Supervised]{\begin{overpic}[width=0.40\columnwidth]{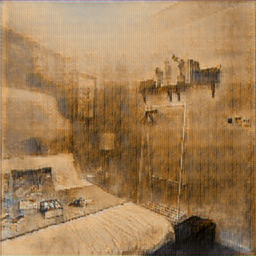}
    \put (5,5) {\scriptsize\textcolor{white}{\texttt{PSNR:16.00}}}
    \put (55,5) {\scriptsize\textcolor{white}{\texttt{SSIM:0.54}}}
\end{overpic}}
\subfigure[CycleGAN]{\begin{overpic}[width=0.40\columnwidth]{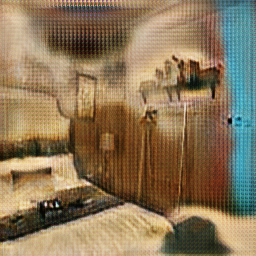}
    \put (5,5) {\scriptsize\textcolor{white}{\texttt{PSNR: 13.09}}}
    \put (55,5) {\scriptsize\textcolor{white}{\texttt{SSIM:0.35}}}
\end{overpic}}
\subfigure[$\text{SUD}^2$]{\begin{overpic}[width=0.40\columnwidth]{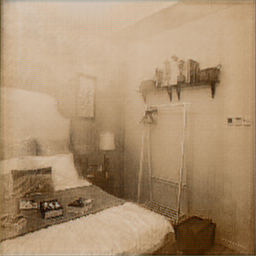}
    \put (5,5) {\scriptsize\textcolor{white}{\texttt{PSNR:14.14}}}
    \put (55,5) {\scriptsize\textcolor{white}{\texttt{SSIM:0.62}}}
\end{overpic}}
\subfigure[Reference]{\includegraphics[width=0.40\columnwidth]{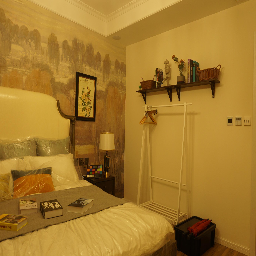}}\\
\vskip 0.1in
\caption{{\bf Dehazing results.} Inference results comparing our method against baselines on dehazing REVIDE bedrooms. All methods use the same network architecture (U-net). All methods are provided 5 hazy / clear image pairs, 200 unpaired hazy images.}
\label{ref:dehazing}
\end{figure*}

\subsection{Dehazing}
\label{sec:dehazing}

\subsubsection{Haze forward model}

Image dehazing is a challenging imaging inverse problem, where the objective is to remove degradations caused by small particulates in the air which obstruct visibility. Typically, dehazing methods approximate the hazy image formation as an depth-dependent attenuation process, defined by the well-known dichromatic atmospheric scattering model~\cite{nayar,vision_bad_weather}
\begin{equation}
    \label{eq:haze}
    I(x,y)=J(x,y)t(x,y)+A(1-t(x,y)).
\end{equation}
Each pixel in the hazy image $I$ is obtained by attenuating pixels in the clear image $J$ according to a transmission map $t$ dependent on depth $d$. Two additional scalar parameters $\beta$ and $A$ describe the magnitude of attenuation and intensity of atmospheric illumination respectively. 

Many prior works such as DehazeNet~\cite{dehazenet,cnn_dehaze} incorporate the atmospheric scattering model in a deep learning framework to train a dehazing network on synthesized hazy images. However, as shown by~\cite{O-HAZE_2018}, these methods often fail when applied to images of real haze---indicating a poor fit to the true forward model of haze.

\subsubsection{Removing real haze}

Currently, few methods exist for learning to dehaze exclusively on images of real haze. CycleDehaze~\cite{cycledehaze} trained on real hazy images from the O-HAZE~\cite{O-HAZE_2018} and I-HAZE~\cite{I-HAZE_2018} benchmarks; however, the low number of hazy samples---less than 100---in those datasets, compounded with the data-hungry nature of neural networks, inhibits the dehazing network from producing high-quality reconstructions.

In this experiment, we aim to overcome these limitations and evaluate our method on a subset of bedroom scenes from the REVIDE dataset~\cite{REVIDE}, which contains a total of 240 real hazy image pairs across 9 scenes. This dataset represents an ideal use case for our method, as only a restricted number of paired images are available to train on.

Similar to the in-painting experiment, we train a U-net on 5 paired images and 200+20,000 unpaired images, where the 5 paired and 200 unpaired images are sourced from REVIDE and the additional 20,000 unpaired images are sourced from the LSUN bedrooms dataset~\cite{lsun}. Before training the image reconstruction network, we first pre-train both a diffusion model and a CycleGAN. Specifically, we train the diffusion model on 20,000 non-hazy bedroom images from the LSUN dataset~\cite{lsun} and train CycleGAN on the same LSUN images (along with the 5 paired and 200 unpaired images from REVIDE). 

Compared to baselines methods, we find that SUD$^2$ produces reconstructions with far fewer artifacts and achieves a median LPIPS score $17\%$ lower than the supervised baseline. Visually, the dehazed images generated by SUD$^2$ have a more natural and smooth appearance relative to both CycleGAN and the supervised baseline as shown in Figure~\ref{ref:dehazing}. In contrast, due to the limited number of hazy training samples, CycleGAN fails to learn a good mapping between hazy and non-hazy images, resulting in poor PSNR, SSIM, and LPIPS scores.

\section{Limitations and future work}
Our experiments demonstrate that SUD$^2$ enables the use of deep learning techniques even when paired data is scarce. Nonetheless, there remain limitations to our approach which can be addressed in future works.

For instance, performance improvements gained from SUD$^2$ are directly correlated with the strength of the prior imposed by the unconditional denoising diffusion model. As the target image distribution grows more diverse, the strength of this learned prior weakens. Conditional diffusion models offer a promising avenue for maintaining a strong prior on diverse image distributions that can be explored in future works.

Furthermore, the reconstruction networks produced by the SUD$^2$ training procedure are deterministic. Extending SUD$^2$ to train stochastic algorithms which produce diverse outputs~\cite{kadkhodaie2021stochastic} is another interesting direction for future work.%

\section{Conclusion}
\label{sec:conclusion}

We introduce SUD$^2$, a generalized deep learning framework for solving few-shot, semi-supervised image reconstruction problems. Inspired by the recent success of denoising diffusion models on image generation tasks, we leverage diffusion models to regularize network training, encouraging solutions that lie close to the desired image distribution. 

To benchmark our method, we apply SUD$^2$ on image reconstruction tasks and compare against CycleGAN, a well-studied generative model that utilizes unpaired data to perform image-to-image translation. When applied to image in-painting, we find that our method produces significantly more structured faces than baselines, where each facial feature is proportionately sized and reasonably positioned on the face. Similarly, when applied to image dehazing, we observe that results from our method are far less noisy compared to baseline methods. Based off of both qualitative and quantitative experimental results, SUD$^2$ succeeds in enabling the training of deep networks on datasets where few paired data samples are available.

\section{Acknowledgements}
M.C. and C.M. were supported in part by the AFOSR Young Investigator Program Award FA9550-22-1-0208 and a Northrop Grumman seed grant.

{\small
\bibliographystyle{ieee_fullname}
\bibliography{main}
}

\newpage

\ifarXiv
    

\section*{Supplementary material}

# Supplemental Materials for
# SUD$^2$: Supervision by Denoising Diffusion Models for Image Reconstruction

Matthew A. Chan
University of Maryland, College Park
mattchan@umd.edu

Sean I. Young
Massachusetts Institute of Technology
siyoung@mit.edu

Christopher A. Metzler
University of Maryland, College Park
metzler@umd.edu

*In this supplementary document, we provide additional information and experimental results including: (1) a detailed description of the network architectures used throughout our experiments, (2) an in-depth exploration of the relative importance of the denoiser loss, correlation minimization loss, and noise injection proposed in the main paper, and (3) images illustrating the benefit that diffusion models have on our reported results.*

## 1. Network architecture

### 1.1. Image reconstruction network

Initially proposed by Ronneberger et al. for biomedical image segmentation, the U-net architecture [9] is a fully convolutional network designed to train efficiently on few annotated data samples. We adopt this architecture as the backbone for our image reconstruction network and implement the contracting and expanding path with 5 encoder an decoder blocks respectively.

Each block contains two convolutional layers—with batch normalization [7] and leaky ReLU activation—stacked together (see Figure S1). Convolutional layers have kernel size 3 and 1 pixel padding. Additionally, leaky ReLU activation functions with a negative slope of 0.01 are used throughout. We perform downsampling along the contracting path via convolutions with a stride of 2. Likewise, upsampling along the expanding path is achieved via transposed convolutions with a stride of 2.

Skip connections [3] concatenate feature maps from corresponding encoder and decoder blocks together, effectively combining fine-grain spatial information from the contracting path with feature-rich information in the expanding path. We propagate spatial information through skip connections at 4 feature map resolutions, specifically at $128 \times 128, 64 \times 64, 32 \times 32$, and $16 \times 16$ resolution.

### 1.2. MMSE denoiser

At its core, our method revolves around a minimum mean squared error (MMSE) denoiser that regularizes the training of deep networks. We implement the MMSE denoiser by training a network to map between images artificially corrupted with Gaussian noise and images devoid of such noise.

For all our experiments in the main paper, we adopt an auto-encoder backbone identical to the U-net architecture in Figure S1, albeit with skip connections removed. However, although Young et al. suggest against the use of U-net denoisers due to their ability to learn noise subtractively via skip connections [14], we observe similar experimental results when using either an auto-encoder backbone or a U-net backbone.

### 1.3. Denoising diffusion models

Similar to the image reconstruction network, the backbone architecture behind our denoising denoising probabilistic model (DDPM) [6] is a U-net. Following Ho et al., the contracting and expanding paths include 6—instead of the 4 in the image reconstruction network—encoder and decoder blocks respectively.

The contracting path consists of the following sequence of encoder blocks: 4 downsampling blocks, 1 downsampling block with spatial self-attention, and a final downsampling block. Mirroring this, the expanding path contains 1 upsampling block followed by an upsampling block with self-attention and 4 additional upsampling blocks. Sinusoidal positional embeddings are used to allow spatial self-attention at the $16 \times 16$ feature map resolution.

In place of batch normalization, we use group normalization [13] in each block and also use sigmoid linear unit (SiLU) [5] activation functions throughout.

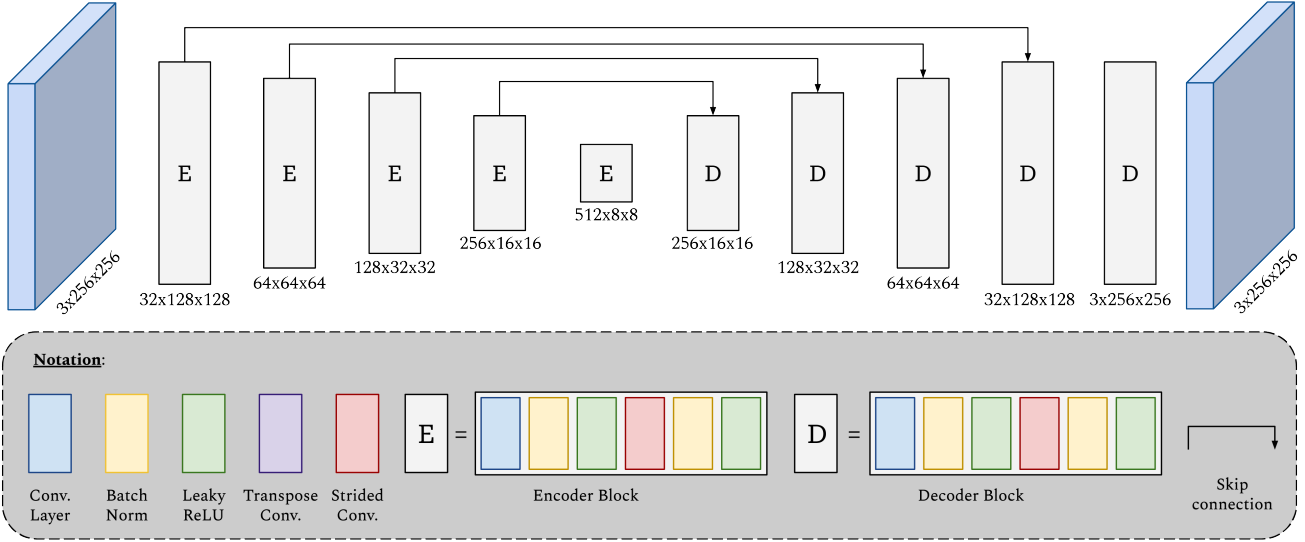

Figure S1. **Image reconstruction network architecture.** The backbone architecture in all experiments is a U-net comprised of 5 encoder and 5 decoder blocks. Output tensor dimensions are provided underneath each block in batch size × height × width format.

### 1.4. CycleGAN

To benchmark our method, we train a CycleGAN on unpaired samples from a source and target data distribution. The architecture of the generator network precisely follows the one introduced by Zhu et al. in [17]. Specifically, the generator is made up of 2 downsampling and upsampling convolutional blocks, implemented with strided convolutions. Between the downsampling and upsampling paths sits 9 residual blocks. Instance normalization [10] and ReLU activations are used throughout.

The architecture behind the discriminator network is a PatchGAN [8] responsible for classifying each $70 \times 70$ pixel patch. We use officially released code from Zhu et al. to train both the generator and the discriminator network.

## 2. Hyper-parameter search

Recall the training objective from the main paper

$$\underset{\theta}{\mathrm{argmin}}\ \mathcal{L}_{\text{paired}} + \lambda_1 \mathcal{L}_{\text{denoiser}} + \lambda_2 \mathcal{L}_{\text{reg}}. \quad (S1)$$

In this section, we coarsely sweep over training hyper-parameters to illustrate the effect that (1) the weight on the denoising loss $\lambda_1$, (2) the weight on the correlation loss $\lambda_2$, and (3) the standard deviation $\sigma$ of injected noise $\nu$ have on image in-painting results.

To do so, we use our best performing hyper-parameter settings (i.e. $\lambda_1 = 0.1, \lambda_2 = 5, \sigma = 4$) as a starting point and individually shift each parameter while observing changes in the reconstructed output.

### 2.1. Denoising loss

The impact that the weight on the denoising loss $\lambda_1$ has on network training is depicted in Figure S2. Keeping all other hyper-parameters fixed, we observe that increases in the regularization strength via $\lambda_1$ cause reconstructed images to collapse to a mode with high probability.

For example, looking at images in row $\lambda_1 = 0.5$, one can see that the faces are nearly identical to one another. Background details from the original masked image such as the subject's hair and neck are all but lost. This loss of detail is correlated with increased weight $\lambda_1$ on the denoising loss, as shown in row $\lambda_1 = 0.2$ where faces are starting to, but have not yet completely, collapsed into a mode.

As stated in the main paper, mode collapse can be mitigated via $\mathcal{L}_{\text{reg}}$; however, increases in $\lambda_1$ counter-act the correlation loss. Consequently, to prevent the collapse caused by increasing $\lambda_1$, the weight on the correlation loss $\lambda_2$ must increase correspondingly.

### 2.2. Correlation minimization

Previously, we introduced a correlation minimizing objective in order to mitigate the likelihood of mode collapse during training. In Figure S3, we show the effect that $\mathcal{L}_{\text{reg}}$ has on the stability of our training pipeline.

One can see that as value of the weight $\lambda_2$ on the correlation loss shrinks, the network converges to a narrower and narrower mode. For example, setting $\lambda_2 = 5$ reconstructs reasonable faces; however, a sufficient reduction in $\lambda_2$ causes these reconstructions to gradually collapse to a mode.

Conversely, boosting the weight on the correlation loss

too high adversely affects latent representations within the network. Since the $\mathcal{L}_{\text{reg}}$ is computed on intermediate feature maps from the U-net's downsampling path, sufficiently large values of $\lambda_2$ begin to perturb the feature maps—visually manifesting as face-like artifacts (e.g. eyes, teeth) at incorrect locations on the face.

### 2.3. Noise injection

As shown in our main paper, denoisers ascend along the gradient of the log-likelihood of true distribution plus some noise, denoted as $p_{x+\nu}$. By injecting Gaussian noise into the network output $f_\theta(y_u)$ prior to denoising it, we more accurately match the distribution of our estimate $p_{f_\theta(y_u)+\nu_2}$ to the true distribution $p_{x+\nu}$. Results in Figure S4 illustrate the effects that this has on network training.

Without injecting noise, training becomes unstable because network outputs before and after denoising are nearly identical—indicating the poor quality in pseudo-labels generated by the denoiser. With noise injection, however, training stabilizes and the denoiser produces pseudo-labelled faces with more prominent features. That being said, adding excessive amounts of noise should be avoided as it swamps $p_{f_\theta(y_u)+\nu_2}$ with noise and causes the denoiser to yield blurrier and blurrier pseudo-labels (see $\sigma = 8$ in Figure S4).

### 2.4. Denoising diffusion models

One interpretation of SUD$^2$ is that the denoiser projects onto the manifold of images it was trained on—in this case images of faces from the CelebAMask-HQ dataset. In doing so, the denoiser generates pseudo-labels for the network to train on.

Inspired by Wakin et al. in [11], we project onto the target image manifold in an iterative coarse-to-fine fashion using denoising diffusion probabilistic models (DDPMs) [6]. By dividing the reverse denoising process into multiple small stages, we force the network to project onto a sequence of increasingly fine manifolds and empirically find that it yields higher-quality pseudo-labels than those generated by a MMSE denoiser.

Improvements over the MMSE denoiser include more symmetric outputs—eyes have similar shape and size relative to one another—and sharper facial features. Among these features, noses and teeth are noticeably less pixelated when diffusion models are used in place of denoisers. See Figure S5 for a direct visual comparison.

### 2.5. Perceptual loss

As noted in [1, 4], MMSE estimators produce reconstructions which are an average of all images on which the estimator was trained, weighted by likelihood. However, this reconstruction is not necessarily a valid one and oftentimes lies outside of the manifold of natural images. We observe this same phenomenon in the blurry reconstructions produced by SUD$^2$ when training with a mean squared error (MSE) loss.

To improve the trade-off between perceptual quality and distortion [1], we incorporate a learned perceptual image patch similarity (LPIPS) [15] loss in both $\mathcal{L}_{\text{paired}}$ and $\mathcal{L}_{\text{denoiser}}$. The inclusion of a perceptual loss—in addition to a mean squared error loss—leads to significantly sharper features and better perceptual quality in the reconstructed images (see Figure S6).

Other auxiliary full-reference loss metrics such as structural similarity index [16], multi-scale structural similarity [12], and deep image structure and texture similarity [2] yield comparable results; however, we find that LPIPS offers the best mix of performance and ease of implementation.

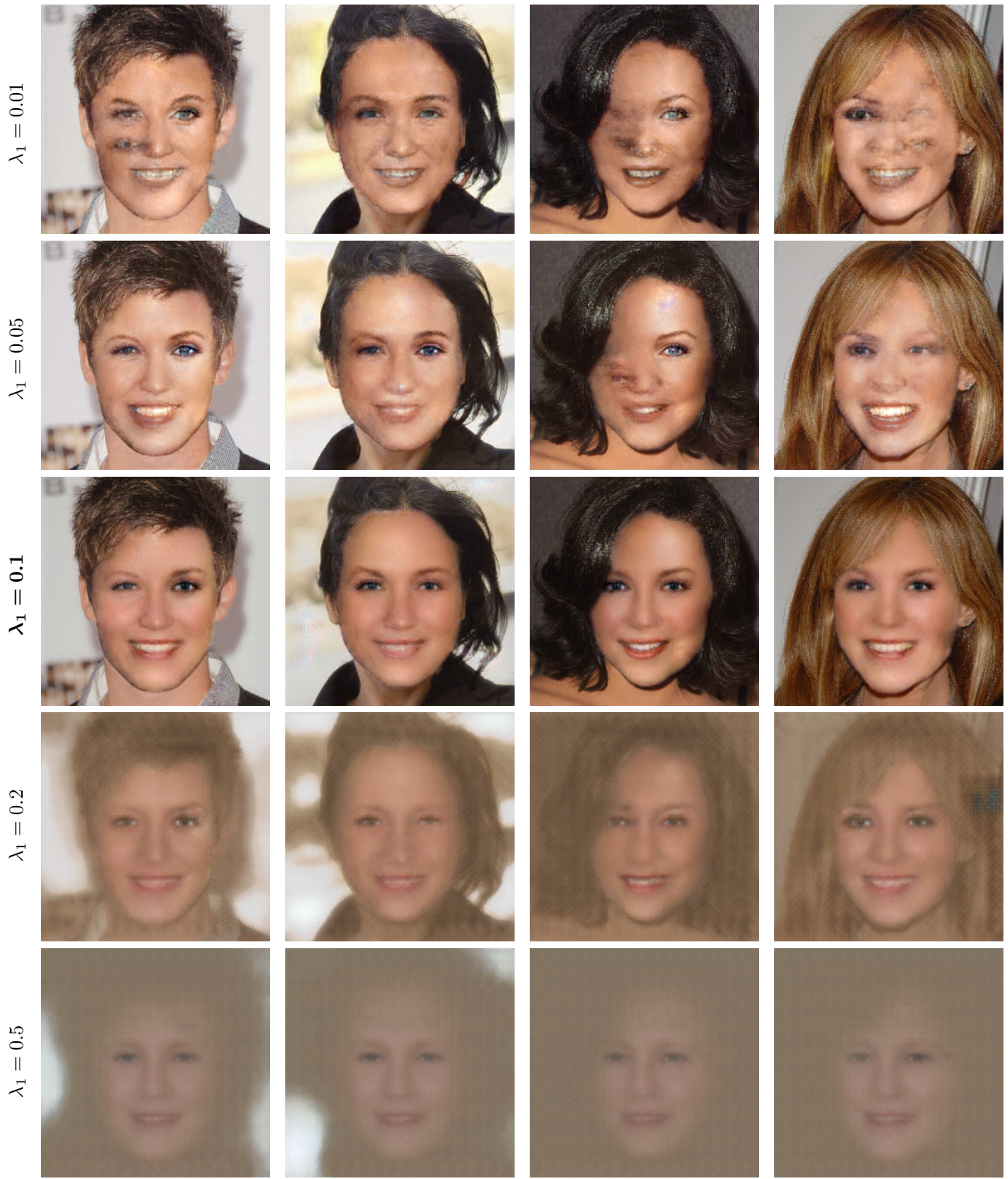

Figure S2. **Hyper-parameter search on the denoising loss weight $\lambda_1$.** Results depict the effect that the weight $\lambda_1$ on the denoising loss weight has on network training. Boosting $\lambda_1$ past a certain threshold causes reconstructions to collapse to a mode. Conversely, shrinking $\lambda_1$ reduces the the regularization strength. In other words, SUD$^2$ results approach the fully supervised training results as $\lambda_1 \to 0$. (Best performing hyper-parameter values are bolded for readability).

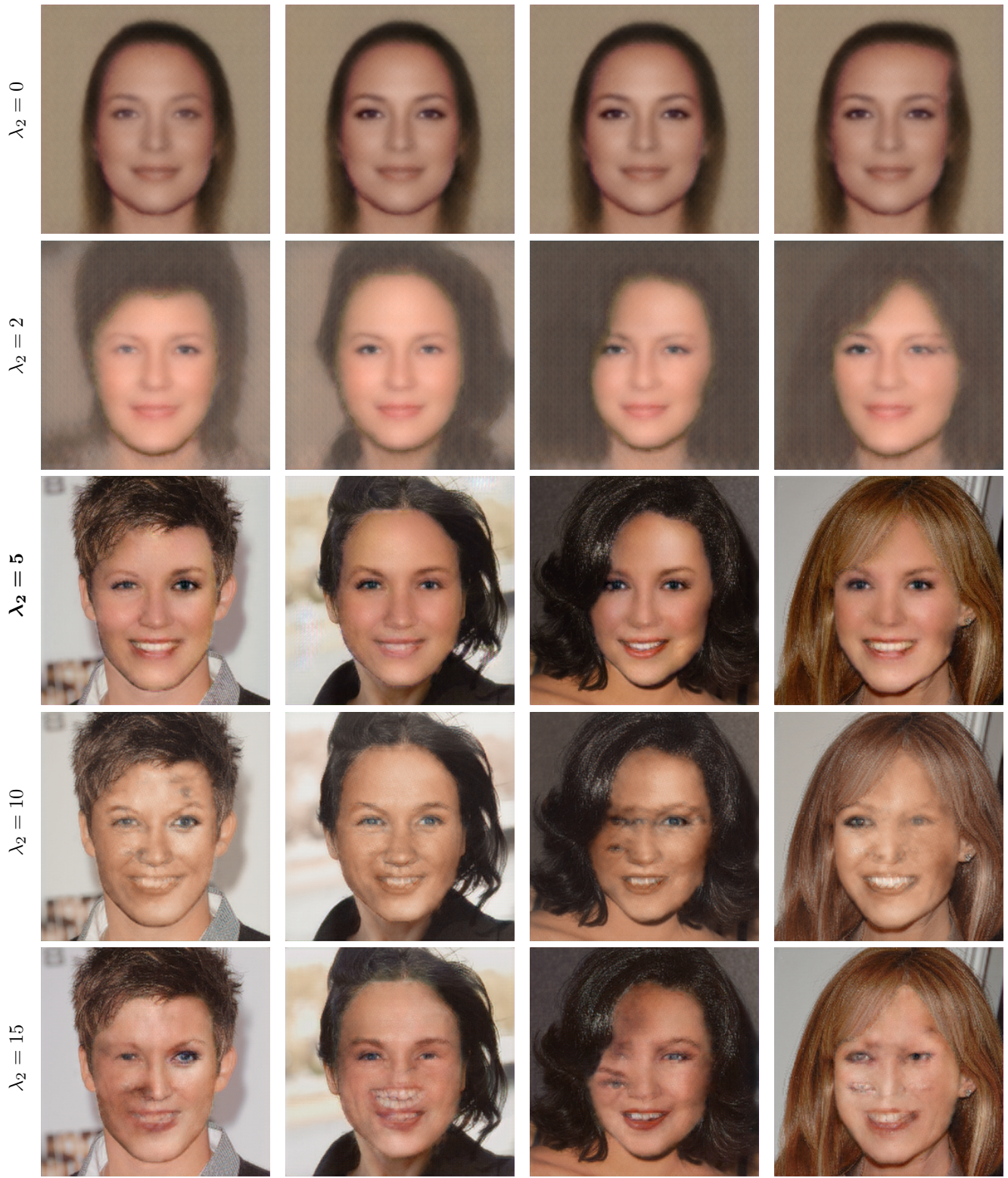

Figure S3. **Hyper-parameter search on the correlation loss weight $\lambda_2$.** Results depict the effect that the weight $\lambda_2$ on the correlation loss has on network training. Decreasing $\lambda_2$ too much leads to mode collapse during training whereas increasing $\lambda_2$ too much causes adverse perturbations in the latent space. (Best performing hyper-parameter values are bolded for readability).

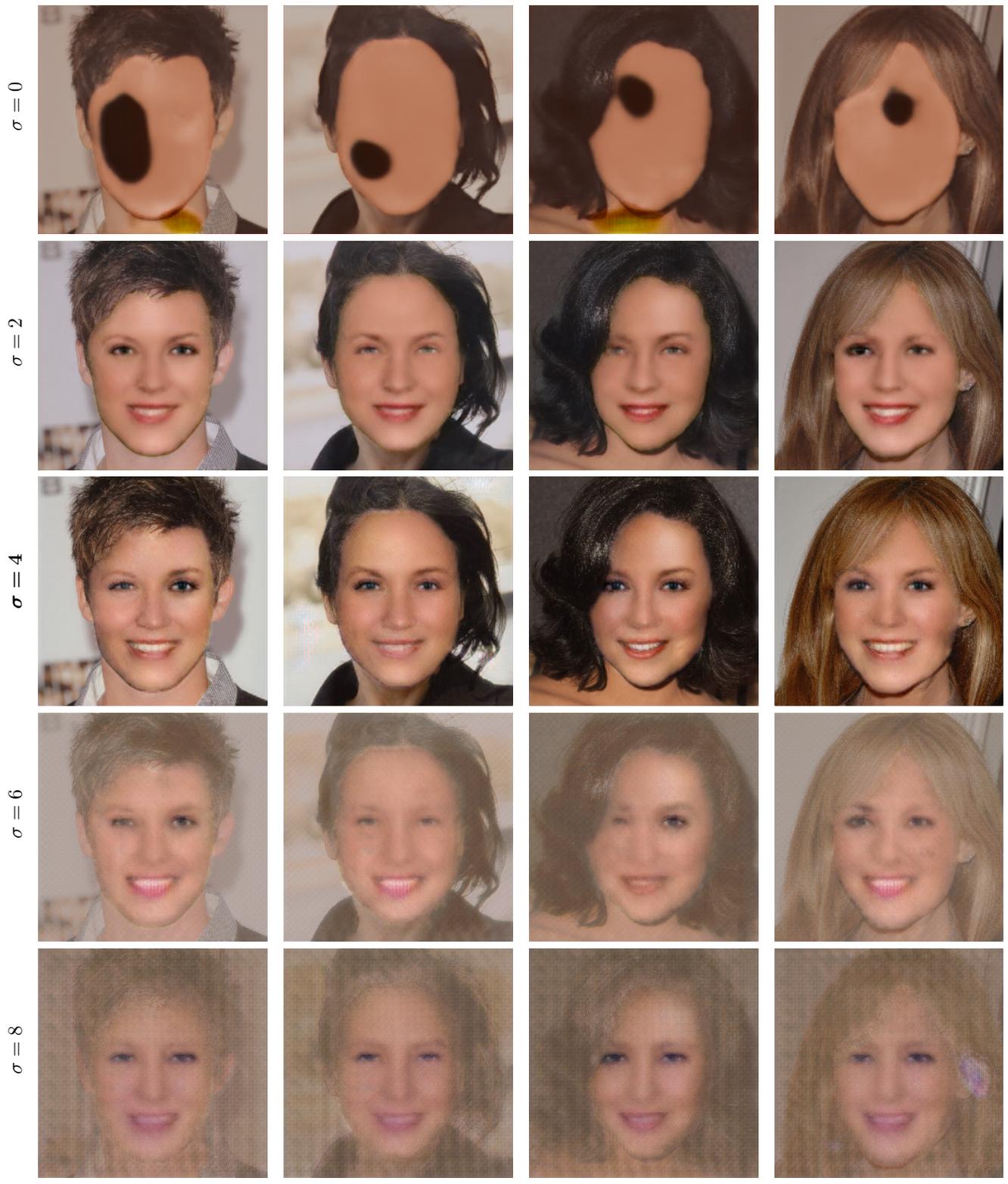

Figure S4. **Hyper-parameter search on the standard deviation $\sigma$ of injected noise $\nu$.** Results depict the effect that the standard deviation $\sigma$ of injected noise has on network training. Turning $\sigma$ up too high destroys the original image signal, resulting in blurry reconstructions. On the other hand, setting $\sigma = 0$ leads to instabilities since the denoising residual is always small. (Best performing hyper-parameter values are bolded for readability).

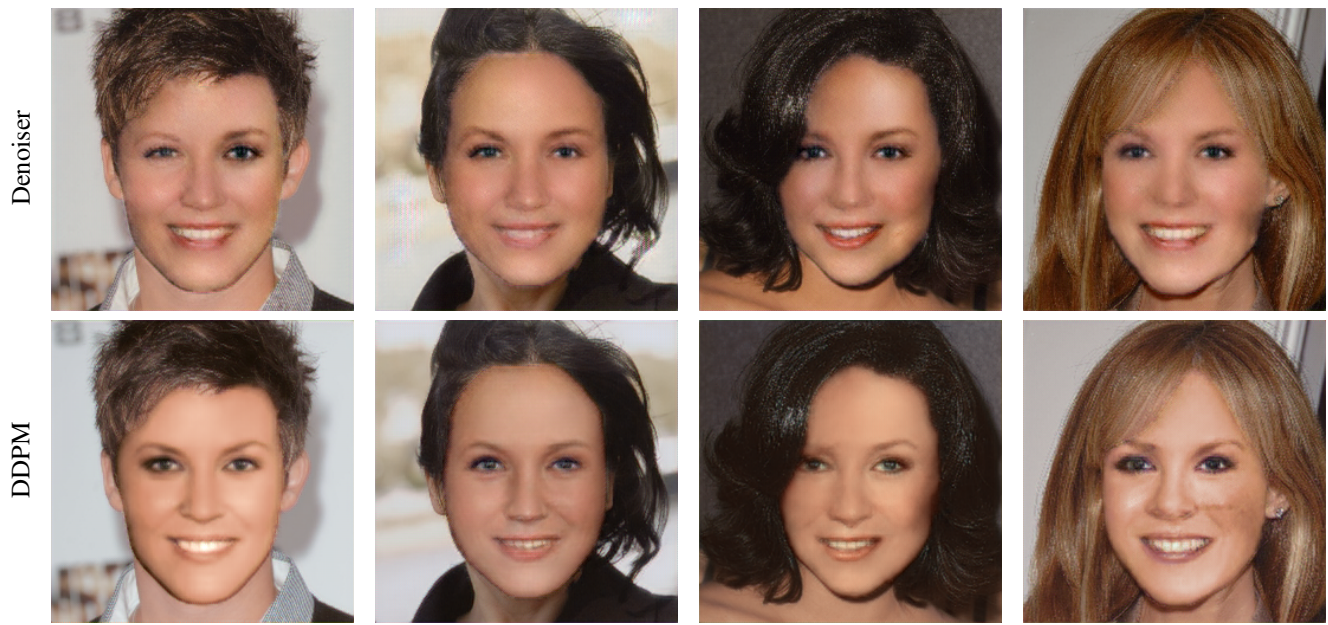

Figure S5. **MMSE denoiser versus diffusion models.** Comparison of results from regularizing training using a denoiser vs. a diffusion model. Training with diffusion models yields more symmetric, higher-quality faces. Facial features (e.g. nose and mouth) also tend to be less blurry.

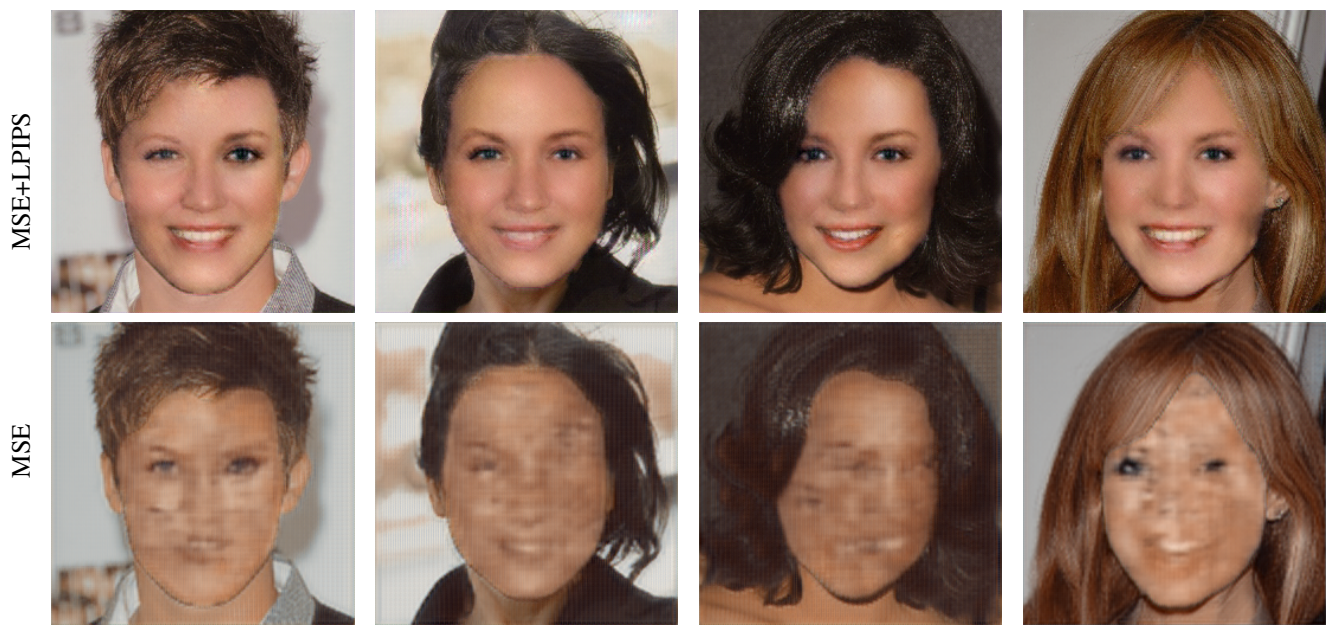

Figure S6. **Effect of perceptual losses.** Comparison of results with and without LPIPS. Inclusion of a perceptual significantly helps the network to fit the data.


\fi

\end{document}